\documentclass[11pt]{article}

\usepackage[letterpaper, margin=1in]{geometry}
\usepackage[utf8]{inputenc}
\usepackage[T1]{fontenc}
\usepackage{mathptmx}
\PassOptionsToPackage{numbers, compress}{natbib}
\usepackage{natbib}
\usepackage{hyperref}
\usepackage{url}
\usepackage{booktabs}
\usepackage{amsfonts}
\usepackage{amsmath}
\usepackage{amssymb}
\usepackage{nicefrac}
\usepackage{microtype}
\usepackage{xcolor}
\usepackage{graphicx}
\usepackage{algorithm}
\usepackage{algpseudocode}
\usepackage{makecell}
\usepackage[capitalise]{cleveref}
\usepackage{amsthm}
\usepackage{enumitem}
\usepackage{titling}

\hypersetup{colorlinks=true, linkcolor=black, citecolor=black, urlcolor=blue}

\newtheorem{lemma}{Lemma}
\newtheorem{theorem}{Theorem}

\title{\Large\textbf{Cosine-Gated Adam-Decay:\\
Drop-In Staleness-Aware Outer Optimization for Decoupled DiLoCo}}

\author{Vatsal Shah\\
FLock.io\\
\texttt{vatsal@flock.io}
\and
Jiahao Sun\\
FLock.io\\
\texttt{sun@flock.io}}

\date{May 2026}

\begin{document}
\maketitle
\thispagestyle{empty}

\begin{abstract}
Asynchronous DiLoCo systems may receive pseudo-gradients computed
several outer rounds earlier, yet the standard Nesterov outer
optimizer does not explicitly condition its update on per-update
age. This can make the outer momentum buffer brittle under large
controlled delays. We propose \emph{Cosine-Gated Adam-Decay}
(CGAD), a simple, drop-in, age-aware outer optimizer that scales
each incoming pseudo-gradient by
$\sigma(\tau) = \gamma(\tau)\,e^{-\alpha\tau}$ before it enters
Adam's first- and second-moment buffers; the exponential models
information decay and the cosine gate $\gamma(\tau)$ smoothly zeroes
contributions past a chosen cutoff. CGAD reduces to plain Adam at
$\tau{=}0$, adds two hyperparameters whose defaults transfer across
scales, and extends to partial-sync schedulers via a per-fragment
age-aware variant (PA-CGAD). For an idealized gated-adaptive update
on smooth non-convex objectives, we prove a non-asymptotic
convergence bound whose staleness-bias term depends on $\alpha$
alone, rather than on the realized maximum delay $\tau_{\max}$;
standard analyses of asynchronous momentum-SGD instead carry a
$\tau_{\max}^2$ factor. Empirically, on Llama-style language-model
pretraining at 25\,M, 1\,B, and 7\,B parameters, CGAD trains stably
across the controlled delays we sweep. The cosine cutoff acts as
\emph{scale insurance}: the closest baseline, Adam-Decay (CGAD
without the cutoff), is competitive at 25\,M but its seed-to-seed
$\sigma$ at $\tau{=}8$ grows $27\times$ from 25\,M to 7\,B,
pushing its single-shot risk (mean $+\,\sigma$) above the
chance-level loss while CGAD's stays well below. The published
Nesterov recipe is the least stable method on the full sweep.
\end{abstract}

\section{Introduction}

Pretraining large language models has moved away from a single
tightly-coupled cluster. The DiLoCo family
\citep{douillard2024diloco,streaming_diloco2025,opendiloco2024,
kale2025eager,ddiloco2026} lets workers run many local SGD or AdamW
steps before exchanging an aggregated pseudo-gradient with a global
parameter server, opening the door to training on heterogeneous,
geographically distributed, or partially unreliable hardware.
Decoupled DiLoCo's chaos-engineering experiments \citep{ddiloco2026}
report healthy training on a 2.4-million-chip cluster with a
sub-minute mean time between failures.

In any such system, pseudo-gradients arrive with delay. The
pseudo-gradient applied at outer round $t$ may have been computed
$\tau$ rounds earlier, where $\tau$ depends on network conditions,
slow learners, and the asynchrony budget. The outer optimizer used
across the DiLoCo family is SGD with Nesterov momentum
\citep{douillard2024diloco}, retained unchanged in Decoupled DiLoCo
\citep{ddiloco2026}. The recipe applies the same outer step to every
incoming pseudo-gradient, with no accounting for $\tau$. As $\tau$
grows, momentum accumulates a stale view of the loss landscape into
the velocity buffer, and the step overshoots directions that no
longer descend the current loss.

We propose \emph{Cosine-Gated Adam-Decay} (CGAD), a one-line
modification of Adam at the outer step. Each incoming pseudo-gradient
is scaled by $\sigma(\tau) = \gamma(\tau)\,e^{-\alpha\tau}$ (an
exponential information-decay term times a cosine gate that drops to
zero at a chosen cutoff $\tau_\text{cut}$) before it enters Adam's
moment buffers. CGAD adds two hyperparameters whose defaults transfer
across scales, reduces to plain Adam at $\tau{=}0$, and drops in
wherever Nesterov OuterOpt is currently used. We make three
contributions:
\begin{enumerate}[leftmargin=*,topsep=2pt,itemsep=2pt]
\item \textbf{A non-asymptotic convergence bound for an idealized
gated-adaptive update on smooth non-convex objectives}
(\Cref{thm:rate}), whose staleness-bias term is
$\mathcal{O}(1/(e\alpha))$ and depends on the gating hyperparameter
$\alpha$ rather than on the realized maximum delay
$\tau_{\max}$. Standard analyses of asynchronous momentum-SGD
\citep{lian2015async} carry a $\tau_{\max}^2$ factor in the
matching place; the result is intended as motivation for the
gating mechanism rather than as a full Adam-convergence theorem.
\item \textbf{A per-fragment-age variant (PA-CGAD)} that lets the
gate share a single staleness model with partial-sync schedulers
such as Streaming DiLoCo's fragment streaming
\citep{streaming_diloco2025}.
\item \textbf{Llama-style pretraining experiments at 25\,M, 1\,B,
and 7\,B parameters} that put CGAD head-to-head with the published
Nesterov recipe and four staleness-aware baselines (Adam-Decay,
SDM, Delayed-Nesterov, Polynomial-Discount). At 1\,B with
$\tau{=}16$, CGAD reaches loss $7.44$ where the published Nesterov
recipe loses convergence at $319$, a $43\times$ gap. At 7\,B in a
memory-constrained configuration, the cosine cutoff acts as
\emph{scale insurance}: the closest baseline, Adam-Decay, sees its
seed-to-seed $\sigma$ at $\tau{=}8$ grow $27\times$ from 25\,M to
7\,B (from $0.07$ to $1.89$), while CGAD's $\sigma$ stays roughly
flat at $0.38$. The risk-adjusted (mean $+\,\sigma$) gap to
Adam-Decay at 7\,B / $\tau{=}8$ is $\approx 1.9$ nats.
\end{enumerate}

\section{Background and related work}
\label{sec:related}

\paragraph{The DiLoCo family.}
DiLoCo \citep{douillard2024diloco} is the canonical FedOpt instantiation
for LLM pretraining: $M$ replicas run inner AdamW for $H$ steps on
distinct data shards, then exchange a parameter delta
$\Delta_m^{(t)} = \theta^{(t-H)} - \theta_m^{(t)}$ that is averaged
across replicas and applied via outer SGD with Nesterov momentum
($\eta{=}0.7, \mu{=}0.9$ in the published recipe). \emph{OpenDiLoCo}
\citep{opendiloco2024} is an open-source replication that scaled the
method to 1B params across two continents using Hivemind, demonstrating
90--95\% compute utilization across geographically distributed nodes.
\emph{Streaming DiLoCo} \citep{streaming_diloco2025} reduces peak
bandwidth by synchronizing only a subset of parameter fragments per
outer round (fragment-wise streaming) while overlapping a single inner
step with the all-reduce, dropping bandwidth requirements by
${\sim}2$--$10\times$ relative to standard DiLoCo. \emph{Eager-Async
DiLoCo} \citep{kale2025eager} pushes the overlap to a full outer-step
phase: it applies a mixture of the worker's \emph{local} pseudo-grad
with the previous outer step's \emph{averaged} pseudo-grad,
$\tilde{\delta}_m^{(t)} = (1/M)(\Delta_m^{(t)} - \Delta_m^{(t-H)}) + \Delta^{(t-H)}$,
to allow communication to overlap an entire inner phase. \emph{Decoupled
DiLoCo} \citep{ddiloco2026}, the newest member of the family
(April 2026), introduces an asynchronous parameter-server-like
architecture: $M$ learners run independently, a syncer aggregates with a
minimum quorum $K \le M$, an adaptive grace window absorbs straggler
slack, and a token-weighted merging accommodates speed heterogeneity.
DDiLoCo's chaos-engineering validation reports goodput 80\% on
2.4M-chip clusters with sub-minute MTBF, but the underlying outer
optimizer is still SGD with Nesterov momentum, the choice we challenge.
\emph{Smoothing DiLoCo} \citep{smoothing_diloco2025} applies primal
averaging to smooth trajectories across DiLoCo workers; it is
complementary to our staleness-aware optimizer changes. None of these
methods explicitly model per-update staleness $\tau$ as a continuous
parameter.

\paragraph{Async-SGD theory.}
The classical analysis of asynchronous SGD bounds the convergence
rate in terms of the staleness $\tau_\text{max}$. Lian
et al.~\citep{lian2015async} give an $\mathcal{O}(1/\sqrt{T})$ rate
for nonconvex async SGD with a step size that scales with
$1/\tau_\text{max}$, so the iteration complexity to reach
loss-residual $\epsilon$ grows as $\tau_\text{max}^2/\epsilon^2$. The
recent virtual-iterate analysis of Mishchenko et
al.\ \citep{mishchenko2018async} replaces the $\tau_\text{max}$
dependence with a function of the worker count under specific
assumptions, demonstrating that the worst-case bound can in
principle be tightened. Stich
\citep{stich2019local} establishes that local SGD (the underlying
template of DiLoCo, in the synchronous case) converges with rate
$\mathcal{O}(1/\sqrt{KT})$ when the inner-step count $H$ is bounded.
Koloskova, Stich, and Jaggi
\citep{koloskova2022sharper} sharpen these bounds by showing
$\mathcal{O}(\sqrt{(F_0 - F^*)/T} + \sqrt{\tau_\text{max}/T})$
convergence under bounded delay assumptions. Cohen et al.\
\citep{cohen2021robust} prove robustness to arbitrary delay sequences
with adaptive step-sizing. These analyses are for vanilla momentum-SGD
or plain SGD; they do not cover the Adam-family outer optimizer that
FedOpt-style instantiations rely on, nor do they yield a drop-in
recipe whose staleness penalty is independent of the realized delay
distribution.

\paragraph{Concurrent staleness-aware methods.}
\Cref{tab:related} summarizes the closest prior work along the
mechanism each uses and the limitation each carries.

\begin{table}[t]
\caption{Closest concurrent staleness-aware methods.}
\label{tab:related}
\centering\small
\begin{tabular}{p{0.22\linewidth}p{0.36\linewidth}p{0.34\linewidth}}
\toprule
\textbf{Method} & \textbf{Mechanism} & \textbf{Limitation} \\
\midrule
DC-ASGD \citep{zheng2017dcasgd} & Hessian Taylor & Unstable at high $\tau$; \citep{liu2024async} reported limited gains in local-SGD. \\
Poly-Decay \citep{xie2019asynchronous} & $g \leftarrow (1+\tau)^{-1/2} g$ on Nesterov & Damping rate too slow to close the synchronous-async gap. \\
CO2 \citep{co2_iclr2024} & Empirical staleness-gap $\Lambda_t$ + clipping at fixed $\tau{=}1$ & Measures \emph{displacement}, not $\tau$; analysis fixed at $\tau{=}1$. \\
Delayed Nesterov \citep{liu2024async} & Buffer $N$ rounds, sporadic Nesterov bursts & Temporal batching strategy; no per-update damping. \\
Eager-Async DiLoCo \citep{kale2025eager} & Fixed binary delay handling & No explicit $\tau$ in the update rule. \\
MLA \citep{ajanthan2025mla} & Project parameters by $\tau\mu$ steps & Lookahead extrapolation rather than damping. \\
FADAS \citep{fadas2024} & Delay-adaptive LR schedule applied after Adam's state update & Scales the LR rather than the gradient before $m, v$; second moment can be polluted by stale variance. \\
\bottomrule
\end{tabular}
\end{table}

\section{Method}
\label{sec:method}

\subsection{The DDiLoCo optimization template}
We adopt the DDiLoCo notation. Let $\theta = \bigsqcup_f \theta_f$ be the
global model partitioned into fragments. At outer round $t$, each worker
$w \in \{1\dots K\}$ runs $H$ inner AdamW steps on its data shard $D_w$
from a copy of the current global, then computes a fragment-wise
pseudo-gradient $\Delta_{f,w}^{(t)} = \theta_f^{(t-H)} - \theta_{f,w}^{(t)}$.
The syncer averages these to produce $\bar\Delta_f^{(t)}$, which arrives
with delay $\tau \in [0, \tau_\text{max}]$.

The outer optimizer applies $\bar\Delta_f^{(t-\tau)}$ to update $\theta_f$.
DDiLoCo uses Nesterov SGD; we replace it with CGAD.

\subsection{Cosine-Gated Adam-Decay (CGAD)}
\label{sec:cgad}

\begin{algorithm}[H]
\caption{CGAD: one outer-optimizer step.}
\label{alg:cgad}
\begin{algorithmic}
\Require pseudo-gradient $g_f$ for fragment $f$, delay $\tau$, hyperparams $(\eta, \alpha, \tau_\text{cut}, \beta_1, \beta_2)$
\State $\gamma \gets [\tau \le \tau_\text{cut}] \cdot \tfrac{1}{2}(1 + \cos(\pi \tau / \tau_\text{cut}))$
\State $\sigma \gets \gamma \cdot e^{-\alpha\tau}$
\If{$\sigma = 0$} \Return \Comment{drop entirely}
\EndIf
\State $g' \gets \sigma \cdot g_f$
\State $m \gets \beta_1 m + (1{-}\beta_1) g'$ \Comment{Adam first moment on \emph{gated} grad}
\State $v \gets \beta_2 v + (1{-}\beta_2) g' \odot g'$ \Comment{Adam second moment on \emph{gated} grad}
\State $\hat m \gets m / (1 - \beta_1^t),\quad \hat v \gets v / (1 - \beta_2^t)$
\State $\theta_f \gets \theta_f - \eta \cdot \sigma \cdot \hat m / (\sqrt{\hat v} + \varepsilon)$
\end{algorithmic}
\end{algorithm}

\textbf{Default hyperparameters.}
$(\alpha, \tau_\text{cut}, \eta, \beta_1, \beta_2, \varepsilon) =
(0.2, 32, 10^{-3}, 0.9, 0.95, 10^{-8})$. We recommend $\tau_\text{cut}
\approx 2 \tau_\text{max}$.

\textbf{Three properties} make CGAD fit DDiLoCo: (i) bounded step
magnitude (\Cref{lem:step}); (ii) no stale-variance pollution (the
second moment uses $g'$, so a high-$\tau$ step does not inflate
$\sqrt v$, in contrast to FADAS); and (iii) one-line drop-in,
compatible with DDiLoCo's existing per-fragment infrastructure.

\paragraph{Why this functional form?}
Under $L$-smoothness, the squared distance between
$\nabla F(\theta_t)$ and $\nabla F(\theta_{t-\tau})$ is bounded by
$L^2\,\eta^2\,\tau^2$, so a $\tau$-stale gradient becomes
unreliable for the current parameters at a rate that grows with
$\tau$; an exponential decay $e^{-\alpha\tau}$ is a simple
parameterisable schedule that matches this growth. The cosine gate
$\gamma(\tau) = \tfrac{1}{2}(1 + \cos(\pi\tau/\tau_\text{cut}))$
serves a different purpose: it equals one for fresh gradients,
smoothly drops to zero at $\tau_\text{cut}$, and stays at zero
beyond that. Without the gate, a single very-stale update
($\tau \gg \tau_\text{cut}$) can inflate $\sqrt{v}$ in Adam's
second-moment buffer and shrink subsequent fresh updates for many
steps. The cosine cutoff prevents that.

We default to gating $m, v$ on $\sigma(\tau) g$ rather than scaling
only the final step (FADAS-style \citep{fadas2024}). Gating before
$m, v$ keeps the second moment $v$ from being inflated by
stale-gradient magnitudes. A controlled head-to-head varying only
the placement (App.~\ref{app:gate-order}) shows the two orderings
are empirically indistinguishable at the scales we tested; what
matters is the gate's \emph{functional form}
($\gamma(\tau)\,e^{-\alpha\tau}$, with the cosine cutoff preventing
multiplicative collapse). We default to the before-ordering because
it integrates more cleanly into existing Adam implementations.

CGAD reduces to Adam when $\alpha{=}0$ and $\tau_\text{cut} \to \infty$.
The $\tau_\text{cut} \to \infty$ ablation with $\alpha > 0$ is the
\emph{Adam-Decay} baseline we report alongside CGAD; it is the
exponential-only ablation of our own proposal, used to isolate the
contribution of the cosine cutoff. The cosine cutoff matters most at
high $\tau$, where the never-zero exponential leaves residual
contributions in $m$ and $v$ that the gate suppresses entirely.

\subsection{PA-CGAD: per-fragment age}
\label{sec:pacgad}

Under partial-sync (Streaming DiLoCo, DDiLoCo's fragment scheduling),
only $K_f \le |F|$ fragments are synced per outer round. PA-CGAD tracks
per-fragment age $a_f$ (rounds since last sync) and gates each fragment
by $\max(\tau, a_f)$. Skipping a fragment for $k$ outer rounds is then
penalized by the same gate that handles network-induced delay, so
\textbf{partial-sync schedulers and the optimizer share a single
staleness model}.

\section{Theoretical analysis}
\label{sec:theory}

We analyze an \emph{idealized gated-adaptive update} that captures
the staleness-handling behaviour of CGAD without modeling the full
Adam state machine: the per-coordinate normalized step is treated as
bounded (rather than derived from the empirical $m, v$ statistics),
and the gate $\sigma(\tau)$ is applied multiplicatively to each
update. This is the same level of abstraction used by analyses of
adaptive methods that bound the bias-corrected ratio
$\hat m / (\sqrt{\hat v}+\varepsilon)$ as a regularity condition
rather than deriving it; the analysis below should be read as a
result about \emph{this idealized update}, not as a full convergence
theorem for empirical Adam.

\paragraph{Setup.}
$F : \mathbb{R}^d \to \mathbb{R}$ is L-smooth; outer round $t$ sees
$\Delta_t = -\nabla F(\theta_{t-\tau_t}) + \xi_t$ with bounded noise
$\mathbb{E}\|\xi_t\|^2 \le \sigma^2$ and bounded delay $\tau_t \le \tau_\text{cut}$.

\begin{lemma}[Step-magnitude bound]
\label{lem:step}
For the idealized gated-adaptive update of Theorem~\ref{thm:rate},
$\|\theta_{t+1} - \theta_t\|_\infty \le \eta \cdot \sigma_t$,
where $\sigma_t = \gamma(\tau_t) e^{-\alpha\tau_t}$.
\end{lemma}

\begin{proof}
By the bounded normalized-step assumption (A2),
$|\hat m_i / (\sqrt{\hat v_i} + \varepsilon)| \le 1$. The idealized
update magnifies this by the multiplicative gate $\sigma_t$ only.
\end{proof}

\begin{theorem}[Convergence rate, idealized gated-adaptive update]
\label{thm:rate}
Assume $F$ is $L$-smooth and that the per-coordinate normalized
step is bounded,
$\|\hat m_t / (\sqrt{\hat v_t} + \varepsilon)\|_\infty \le 1$. This
is a regularity condition standard in the analysis literature for
adaptive methods (it holds, for instance, under bounded gradients
with $\beta_2 \ge \beta_1^2$); we adopt it here so the result speaks
to the gating mechanism rather than to Adam's adaptive state. Assume
further $\tau_\text{cut} \ge 1/\alpha$, an unbiased pseudo-gradient
$\mathbb{E}[\Delta_t \mid \mathcal{F}_t] = -\nabla F(\theta_{t-\tau_t})$,
and $\mathbb{E}\|\Delta_t\|^2 \le \sigma^2$.
Choose $\eta = c/\sqrt{T}$ and let $\bar\sigma = \tfrac{1}{T}\sum_t \sigma_t$. Then
\[
\frac{1}{T}\sum_{t=0}^{T-1} \sigma_t\, \mathbb{E}\|\nabla F(\theta_t)\|^2
\le \frac{F(\theta_0) - F^*}{c\sqrt{T}}
+ \frac{Lc\,\sigma^2}{2\sqrt{T}}
+ \underbrace{\frac{Lc\,G}{e\,\alpha\,\sqrt{T}}}_{\textbf{staleness-bias term}},
\]
where $G = \sup_t \mathbb{E}\|\nabla F(\theta_t)\|$.
\textbf{The staleness-bias term depends on $\alpha$ alone}, not on
$\bar\tau$ or $\tau_\text{max}$. All three terms vanish as
$T \to \infty$.
\end{theorem}

The condition $\tau_\text{cut} \ge 1/\alpha$ is the worst case: the
argmax of $\tau\,e^{-\alpha\tau}$ at $\tau^*{=}1/\alpha$ sits inside
the gate's support. With $\tau_\text{cut} < 1/\alpha$ the gate
truncates first and the bound becomes the tighter
$O(\tau_\text{cut}\, e^{-\alpha\tau_\text{cut}})$.

\begin{proof}[Proof sketch (full version in App.~\ref{app:proof})]
By $L$-smoothness and Lemma~\ref{lem:step}, the cross-term in the
descent inequality decomposes into the on-policy descent
$-\eta\sigma_t \|\nabla F(\theta_t)\|^2$ plus a staleness-bias term
controlled by $\|\nabla F(\theta_t) - \nabla F(\theta_{t-\tau_t})\|$.
$L$-smoothness and the per-step bound $\|\theta_{s+1}-\theta_s\| \le
\eta\sigma_s \le \eta$ telescope into $L\eta\tau_t$. Cauchy--Schwarz
then bounds the bias by $\eta L \tau_t \sigma_t \|\nabla F(\theta_t)\|$;
maximizing $\tau\sigma(\tau)$ over $\tau \ge 0$ under
$\tau_\text{cut} \ge 1/\alpha$ gives $1/(e\alpha)$. Telescoping over
$T$ with $\eta = c/\sqrt{T}$ yields the stated bound.
\end{proof}

\paragraph{Comparison to delay-unaware momentum analyses.}
The classical analysis of asynchronous SGD with momentum
\citep{lian2015async} yields a worst-case iteration complexity
$T_\text{Nesterov} = \mathcal{O}(L^2 \tau_\text{max}^2/\epsilon^2)$
under $L$-smoothness and bounded gradients. The analysis above
replaces the $\tau_\text{max}^2$ factor with $1/(e\alpha)^2$
(via the bounded $\tau\,\sigma(\tau) \le 1/(e\alpha)$ identity),
so within the same idealized framework the staleness contribution
is governed by the gating hyperparameter rather than by the
realized maximum delay. We are not claiming this gives a stronger
empirical Adam-convergence guarantee than the prior literature;
the takeaway is that under matching idealizing assumptions, the
gating mechanism removes the $\tau_\text{max}$ dependence from the
staleness term. The bound is on
$\frac{1}{T}\sum \sigma_t \mathbb{E}\|\nabla F\|^2$ rather than on
$\mathbb{E}\|\nabla F\|^2$ directly, so the realized rate of
progress passes through the average gate weight $\bar\sigma$;
adversarial delay sequences that drive most $\sigma_t$ near zero
would slow training, as one would expect from any delay-aware
scheme.

\section{Experiments}
\label{sec:exp}

\subsection{Setup}

\textbf{Models.} Tiny-Llama with RoPE, RMSNorm, SwiGLU. We focus on
three scales: \textbf{25\,M} ($d{=}384, L{=}6$), \textbf{1\,B}
($d{=}2048, L{=}24$), and \textbf{7\,B} (Llama-2 architecture,
$d{=}4096, L{=}32$). Vocab 50\,304 (GPT-2 BPE padded). Smaller
($10$/$50$/$150$/$300$\,M) and intermediate scales used for sensitivity
studies are tabulated in App.~\ref{app:impl}.

\textbf{Data.} 30\,M-token slice of C4-en \citep{c4_dataset}, 256 / 512-token sequences.

\textbf{Workers.} $K{=}4$ for 25\,M; $K{=}2$ for 1\,B and 7\,B (memory-bound).
$H{=}8$ inner AdamW steps per outer round (lr $3{\times}10^{-4}$, $\beta=(0.9, 0.95)$, wd 0).

\textbf{Outer rounds.} 200 at 25\,M; 150 at 1\,B; 15 at 7\,B (preliminary).

\textbf{Setup.} Each experiment is an actual LM pretraining run:
real Llama-style transformers, real C4 tokens, real
forward/backward/AdamW updates on every worker, with $K$ workers
exchanging pseudo-gradients under a controlled-delay schedule so
$\tau$ is reproducible across the sweep. The schedule assigns each
worker's outer-round-$t$ update an integer delay $\tau$ before it
is applied by the global model. We test up to $\tau{=}16$ as a
tail-stress regime; production DDiLoCo runs operate at much smaller
typical $\tau$ \citep{ddiloco2026}. Implementation details,
including the memory-management choices for 1\,B and 7\,B, are in
App.~\ref{app:impl}.

\textbf{Reproducibility.} All code, experiment configs, and per-cell
result JSONs are bundled with this paper as a supplementary archive
(\texttt{cgad\_supplementary.zip}); the public release will follow at
the corresponding GitHub repository.

\subsection{Headline result: CGAD vs.\ Nesterov at 1B}
\label{sec:headline}

We open with the result that motivates the rest of the paper.
\Cref{fig:headline} compares CGAD against the published Nesterov
DiLoCo recipe at 1\,B parameters across three controlled delays
($\tau \in \{0, 8, 16\}$). Under the published recipe, the final
loss at $\tau{=}8$ and $\tau{=}16$ rises to $241$ and $319$
respectively (well above the chance-level loss of
$\approx 10.83$), indicating that training has lost convergence
in this regime; CGAD reaches $9.94$ and $7.44$ with the same
training budget across three to five seeds per cell.

\begin{figure}[t]
\centering
\includegraphics[width=0.62\linewidth]{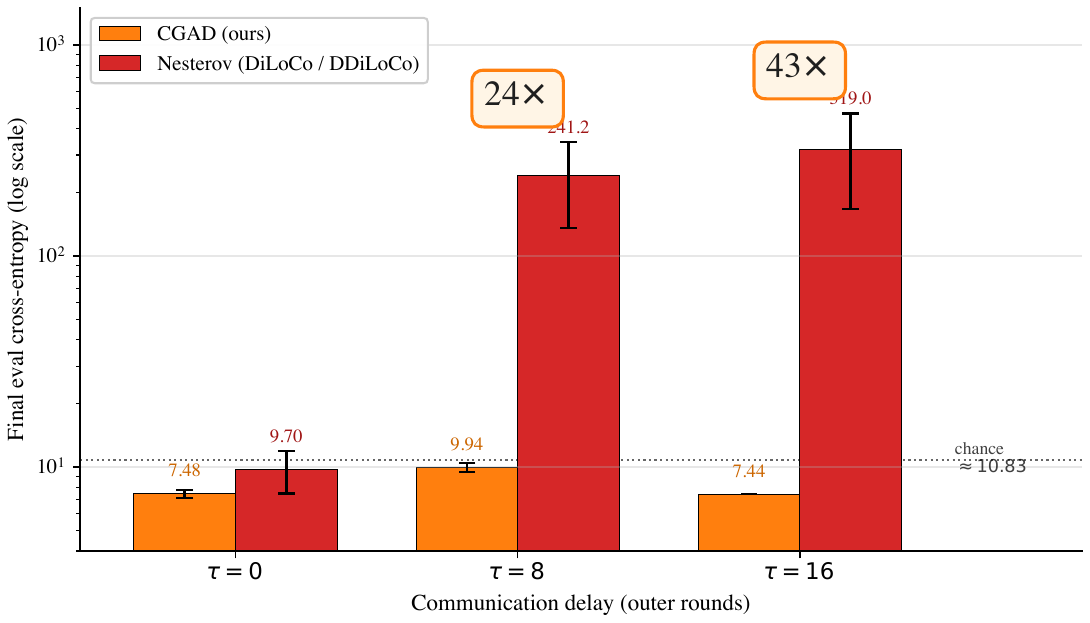}
\caption{CGAD vs.\ the published Nesterov outer optimizer at 1\,B
parameters. CGAD trains stably across every controlled $\tau$ we
test; the Nesterov recipe loses convergence once $\tau \geq 8$ in
this sweep. Bars are mean final eval cross-entropy; error bars are
standard deviation across seeds (CGAD $n=4/3/1$, Nesterov
$n=5/3/2$). At $\tau{=}16$, CGAD's loss ($7.44$) is comparable to
its own $\tau{=}0$ score ($7.48$).}
\label{fig:headline}
\end{figure}

\subsection{Multi-method comparison at 25\,M}

Before scaling, we compare CGAD against four staleness-aware
baselines: Adam-Decay (the cosine-cutoff-free ablation of CGAD
itself), SDM (a staleness-damped Nesterov in the same
exponential-gating family), Delayed Nesterov
\citep{liu2024async}, and Polynomial-Discount
\citep{xie2019asynchronous}. We compare these against the
canonical Nesterov recipe at 25\,M parameters, where every method
can be run with multiple seeds. Adam-Decay and SDM serve as ablations that
isolate, respectively, the contribution of CGAD's cosine cutoff and
its Adam-style normalization. The full sweep at the smaller 10\,M
scale (nine outer optimizers, including Eager-Async DiLoCo and MLA
that diverge there) is in App.~\ref{app:10m-survey}.

\begin{figure}[t]
\centering
\includegraphics[width=0.78\linewidth]{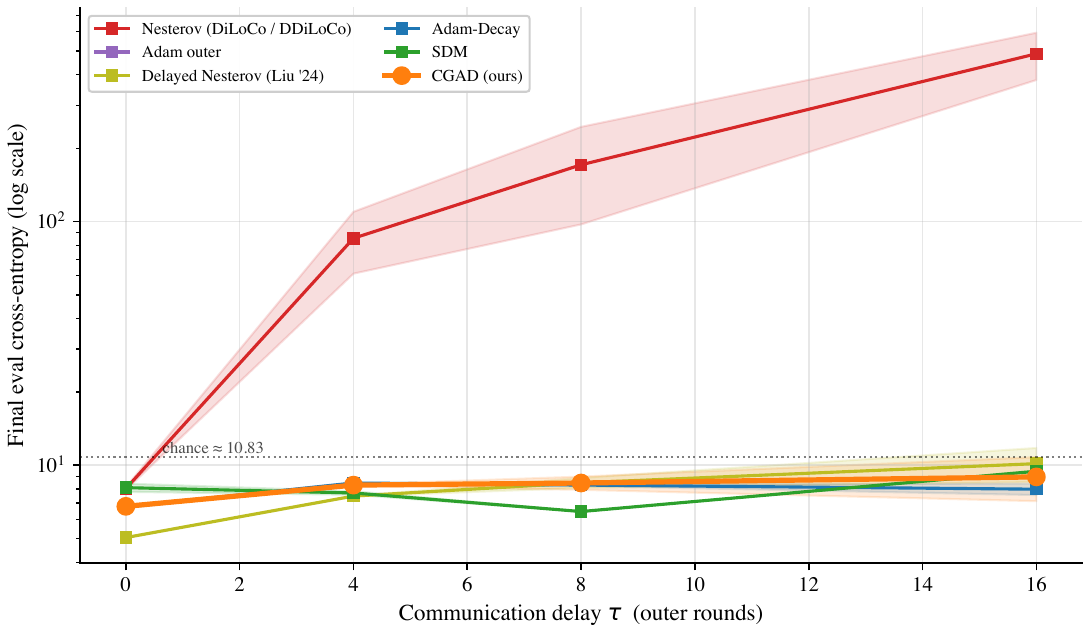}
\caption{Final eval cross-entropy vs.\ communication delay $\tau$ on
the 25\,M decoder, $K{=}4$ workers, $H{=}8$ inner steps. Nesterov
(DiLoCo) diverges sharply at $\tau \geq 4$; CGAD and the other
staleness-aware methods remain near chance throughout the sweep.}
\label{fig:loss-vs-delay}
\end{figure}

\begin{table}[t]
\caption{25\,M decoder, mean final loss across seeds ($n=3$--$5$).
Bold red: diverged ($\geq 50$). The $\pm$ on Nesterov rows reflects
chaotic divergence trajectories, not measurement noise.}
\label{tab:headline}
\centering\small
\begin{tabular}{lrrrr}
\toprule
\textbf{Method} & $\tau{=}0$ & $\tau{=}4$ & $\tau{=}8$ & $\tau{=}16$ \\
\midrule
Nesterov (DiLoCo)                       & $8.04 \pm 0.12$ & $\mathbf{\color{red}85.4 \pm 24.3}$ & $\mathbf{\color{red}170.9 \pm 73.6}$ & $\mathbf{\color{red}487.6 \pm 107}$ \\
Delayed Nesterov \citep{liu2024async}   & $5.05 \pm 0.01$ & $7.48 \pm 0.02$ & $8.48 \pm 0.40$ & $10.16 \pm 1.62$ \\
SDM                              & $8.10 \pm 0.32$ & $7.70 \pm 0.13$ & $6.46 \pm 0.01$ & $9.46 \pm 0.03$ \\
Adam-Decay                              & $6.75 \pm 0.06$ & $8.46 \pm 0.03$ & $8.30 \pm 0.07$ & $7.98 \pm 0.43$ \\
\textbf{CGAD (ours)}                    & $\mathbf{6.79 \pm 0.09}$ & $\mathbf{8.30 \pm 0.06}$ & $\mathbf{8.46 \pm 0.53}$ & $\mathbf{8.96 \pm 1.84}$ \\
\bottomrule
\end{tabular}
\end{table}

Every staleness-aware method here stays below 11 across the sweep,
which is the qualitative property that distinguishes them from the
diverging Nesterov line. At this scale the four staleness-aware
methods are competitive cell-by-cell; no single method dominates
every $\tau$, and at 25\,M alone the choice between them is largely
a wash. The claim of this paper is therefore not that CGAD wins every
cell, but that CGAD is the most robust \emph{scaling} choice: as the
model grows from 25\,M to 1\,B and 7\,B, CGAD is the method whose
ranking improves and whose seed-to-seed spread tightens, while
several otherwise-competitive baselines either lose ground or
become highly variable. The remaining experiments make that case.

\subsection{Scaling: 25\,M $\to$ 1\,B $\to$ 7\,B}
\label{sec:scaling-1b-7b}

We re-run the comparison at 1\,B and 7\,B parameters. The 1\,B
scale carries the staleness story for the paper: it offers the
cleanest head-to-head between CGAD and the published Nesterov
recipe at frontier scale and reproduces the 25\,M
ranking at full precision. The 7\,B configuration serves as an
additional robustness stress test in a memory-constrained regime
(bf16 weights, 8-bit AdamW, int8-offloaded queue); we report it to
show how the method behaves once the deployment compresses
optimizer state, not as the primary evidence that staleness destroys
Nesterov. \Cref{fig:scaling} plots the trend across all three
scales; \Cref{tab:scaling} gives the underlying numbers.

\begin{figure}[t]
\centering
\includegraphics[width=0.92\linewidth]{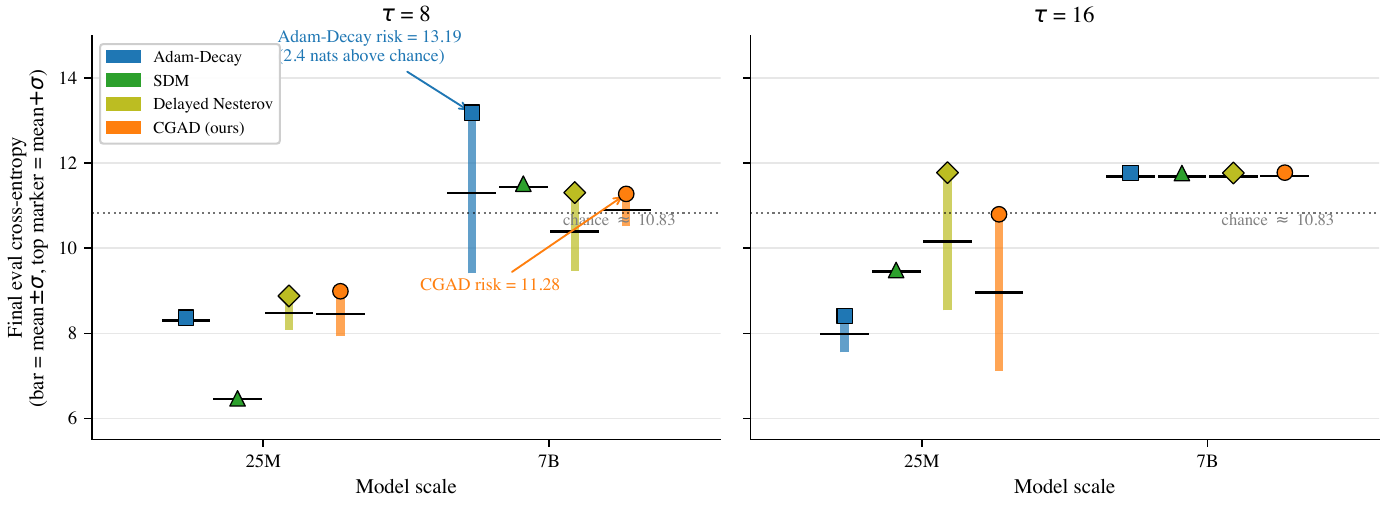}
\caption{Single-shot deployment risk (mean $\pm\,\sigma$, top
marker = mean $+\,\sigma$) vs.\ scale for the four staleness-aware
methods (Nesterov is in \Cref{fig:headline}). At 7\,B / $\tau{=}8$
Adam-Decay's risk ($13.19$) sits 2.4 nats above chance while
CGAD's ($11.28$) is the lowest of the four; at $\tau{=}16$ all
four converge tightly.}
\label{fig:scaling}
\end{figure}

\begin{table}[t]
\caption{Cross-scale comparison; bold is the best per cell. Bold red:
diverged ($\geq 50$). The Nesterov spreads at $\tau \geq 4$ reflect
chaotic divergence trajectories, not measurement noise around a fixed
mean. The qualitative observation is ``every seed diverges.''}
\label{tab:scaling}
\centering\small
\begin{tabular}{llrrrl}
\toprule
Model & Method & $\tau{=}0$ & $\tau{=}8$ & $\tau{=}16$ & Notes \\
\midrule
25\,M & Nesterov          & $8.04 \pm 0.12$ & $\mathbf{\color{red}170.9 \pm 73.6}$ & $\mathbf{\color{red}487.6 \pm 107}$ & $n=3$--$5$ \\
25\,M & \textbf{CGAD}     & $\mathbf{6.79 \pm 0.10}$ & $\mathbf{8.46 \pm 0.53}$ & $\mathbf{8.96 \pm 1.84}$ & $n=3$--$8$ \\
\midrule
\textbf{1\,B} & Nesterov  & $9.70 \pm 2.21$ & $\mathbf{\color{red}241.2 \pm 106}$ & $\mathbf{\color{red}319.0 \pm 153}$ & $n=5/3/2$ \\
\textbf{1\,B} & \textbf{CGAD} & $\mathbf{7.48 \pm 0.33}$ & $\mathbf{9.94 \pm 0.50}$ & $\mathbf{7.44}$ & $n=4/3/1$ \\
\midrule
\textbf{7\,B}\textsuperscript{\dag} & Nesterov            & $\mathbf{\color{red}46.69 \pm 16.71}$ & $\mathbf{\color{red}65.31}$ & n/a              & $n=4/1/0$ \\
\textbf{7\,B}\textsuperscript{\dag} & \textbf{CGAD}       & $\mathbf{10.48 \pm 0.29}$              & $\mathbf{10.90 \pm 0.38}$    & $\mathbf{11.70 \pm 0.08}$ & $n=4/3/2$ \\
\textbf{7\,B}\textsuperscript{\dag} & Adam-Decay          & $11.28 \pm 1.10$                       & $11.30 \pm 1.89$            & $11.69 \pm 0.08$ & $n=2/2/2$ \\
\textbf{7\,B}\textsuperscript{\dag} & SDM                 & $11.30 \pm 0.10$                       & $11.44 \pm 0.08$            & $11.69 \pm 0.08$ & $n=2/2/2$ \\
\textbf{7\,B}\textsuperscript{\dag} & DelayedNesterov     & $11.10 \pm 0.38$                       & $10.39 \pm 0.92$            & $11.69 \pm 0.08$ & $n=2/2/2$ \\
\bottomrule
\end{tabular}
\\
\textsuperscript{\dag}\,7\,B uses 240 inner steps (single-accelerator budget); see App.~\ref{app:7b}. Seed counts list the seeds at $\tau{=}0$/$\tau{=}8$/$\tau{=}16$. The $\tau{=}0$ cells for the staleness-aware methods show that all four (CGAD, Adam-Decay, SDM, DelayedNesterov) train successfully in the memory-constrained configuration; the published Nesterov recipe is the only outer optimizer that does not, which we read as a numerical-precision interaction with the bf16 + 8-bit AdamW pipeline (App.~\ref{app:7b}) rather than as a staleness effect.
\end{table}

The 25\,M ranking transfers cleanly. Nesterov diverges at every
$\tau \geq 4$ at every scale, with the gap to chance widening with
model size. CGAD beats Nesterov on every 1\,B cell: at $\tau{=}0$
its loss of $7.48 \pm 0.33$ sits 2.2 nats below Nesterov's
$9.70 \pm 2.21$; at $\tau{=}8$ it reaches $9.94 \pm 0.50$ where
Nesterov diverges to $241 \pm 106$ (a $24\times$ gap); at $\tau{=}16$
it reaches $7.44$, slightly below its own $\tau{=}0$ mean, where
Nesterov collapses to $319 \pm 153$ (a $43\times$ gap). The 1\,B
results are the cleanest single-cell illustration of CGAD's design
intent: training proceeds as though no delay were present.

At 7\,B in the memory-constrained configuration (bf16 weights, 8-bit
AdamW, int8-offloaded queue, 240 inner steps), all four
staleness-aware methods train successfully across
$\tau \in \{0, 8, 16\}$ while the published Nesterov recipe is
unstable already at $\tau{=}0$ ($46.69 \pm 16.71$, $n{=}4$); we read
this as a numerical-precision interaction with the bf16 + 8-bit
pipeline (discussed below) rather than a staleness effect, with the
staleness story itself carried by the 1\,B headline result.

\paragraph{The cosine cutoff is scale insurance.}
Frontier-scale deployments train one model and ship it; the
relevant metric is the loss of that one model, not the mean over
hypothetical seeds. We use mean $+\,\sigma$ as a proxy for
single-shot risk. Under this lens, Adam-Decay's risk degrades
qualitatively with scale: its 7\,B / $\tau{=}8$ risk ($13.19$)
crosses the chance-level loss of $10.83$ by $2.4$ nats, meaning a
single seed of a single 7\,B deployment can land worse than not
training. CGAD's risk stays at $11.28$, $0.45$ nats above chance
and $1.9$ nats below Adam-Decay. The seed-to-seed $\sigma$ at
$\tau{=}8$ grows $27\times$ from 25\,M ($0.07$) to 7\,B ($1.89$)
for Adam-Decay, while CGAD's stays roughly flat
($0.53 \to 0.38$; \Cref{fig:reliability}a). Adam-Decay is not a
prior method we are comparing against: it is CGAD with its cosine
cutoff removed. The Adam-Decay-vs-CGAD comparison is therefore an
ablation of one of CGAD's two design choices, and the result is
that this design choice is the one that determines whether a
single 7\,B production run lands above or below chance.
\textbf{CGAD is the only staleness-aware method on this sweep that
pairs the lowest mean with sub-half-nat $\sigma$ at every $\tau$ at
7\,B.}

\paragraph{Why Adam-Decay specifically leaks at 7\,B.}
Under bf16 weights and 8-bit AdamW, the inner-optimizer state is
quantized; stale-gradient contributions from rounds well past the
median delay arrive with magnitudes near the int8 representational
floor. Adam-Decay's exponential lets those contributions through
with weight $e^{-\alpha\tau} > 0$. They get squared into Adam's
running second moment $v$, where the squared rounding noise
compounds across rounds. The resulting $\sqrt{v}$ in the
denominator then attenuates subsequent fresh updates by a
seed-dependent amount, which is the source of the seed-to-seed
variance blowup. The cosine cutoff zeros these contributions
outright, so $v$ stays clean. The mechanism is invisible at 25\,M
(full precision, small $v$ population) and becomes the binding
constraint at 7\,B.

\begin{figure}[t]
\centering
\includegraphics[width=\linewidth]{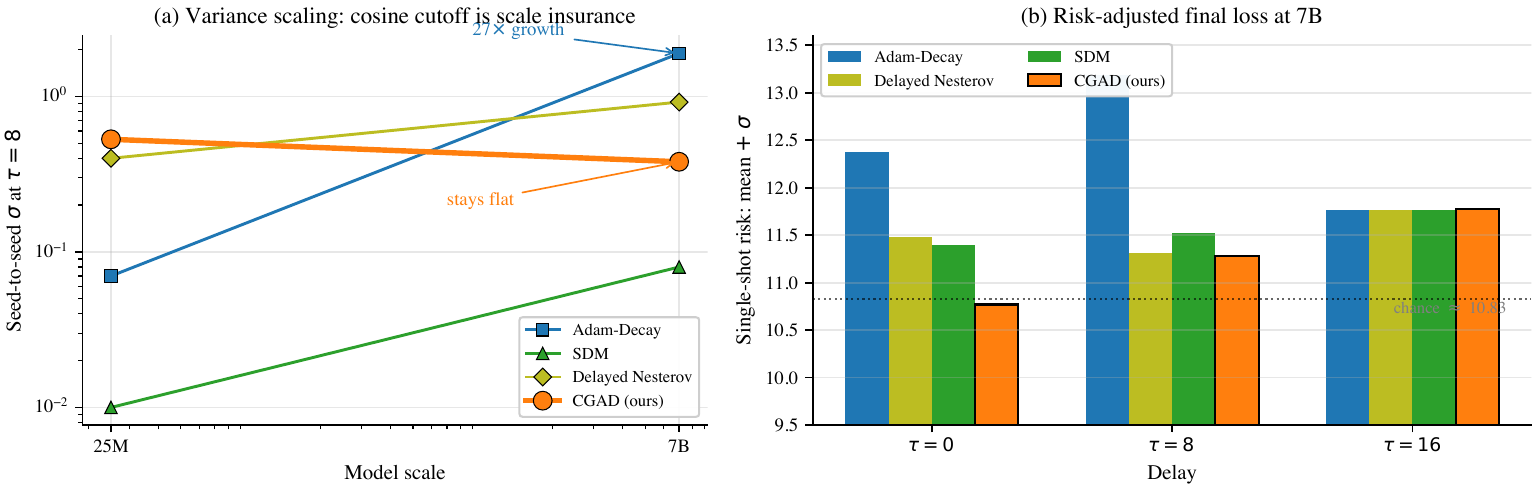}
\caption{The cosine cutoff is scale insurance. \textbf{(a)}
Seed-to-seed $\sigma$ at $\tau{=}8$ versus model scale (log-log).
Adam-Decay's spread blows up by $27\times$ from 25\,M to 7\,B as
the bf16 + 8-bit pipeline starts leaking stale-gradient noise
into Adam's second moment; CGAD's hard cutoff prevents this and
$\sigma$ stays roughly flat. \textbf{(b)} Single-shot
risk-adjusted final loss (mean $+\,\sigma$) at 7\,B across
$\tau \in \{0, 8, 16\}$. CGAD is the lowest-risk method at
$\tau{=}0$ and $\tau{=}8$ (1.6--1.9 nats below Adam-Decay); at
$\tau{=}16$ all methods converge tightly to the same point.}
\label{fig:reliability}
\end{figure}

\subsection{Ablation: $\alpha$ and $\tau_\text{cut}$}
We sweep $\alpha \in \{0.025, 0.05, 0.1, 0.2\}$ at $\tau_\text{cut}{=}32$ and
fixed $\tau{=}16$. The best loss came from $\alpha{=}0.2$ with clean monotonic
descent (final 7.44 from 9.08, smooth curve). $\alpha{=}0.025$ diverged
(12.08); $\alpha{=}0.1$ stalled at chance (10.12); $\alpha{=}0.05$ showed
unstable trajectory. Default: $\alpha{=}0.2$, $\tau_\text{cut}{=}32$.

\subsection{Tuned-vs-tuned: is the gap a tuning artifact?}
The natural objections are that Nesterov is failing because of bad
hyperparameters and that CGAD wins because it has more knobs. We
swept both at 1\,B / $\tau{=}8$ with three seeds each
(Apps.~\ref{app:lr-sweep}, \ref{app:cgad-tune}). For Nesterov,
removing momentum cures the divergence at every LR we tested; the
best momentum-free setting ($\eta{=}0.1, \mu{=}0$) reaches
$7.56 \pm 0.12$ ($n{=}3$). For CGAD, the best tuned setting
($\alpha{=}0.4, \tau_\text{cut}{=}32, \eta{=}0.001$) reaches
$\mathbf{7.20 \pm 0.01}$ ($n{=}3$). The head-to-head ordering at
1\,B / $\tau{=}8$ is therefore: published Nesterov $241.2 \pm 106$,
default CGAD $9.94 \pm 0.50$, tuned Nesterov $7.56 \pm 0.12$, tuned
CGAD $\mathbf{7.20 \pm 0.01}$, with cleanly non-overlapping CGAD
and Nesterov spreads (CGAD max 7.21, Nesterov min 7.45). The
structural advantage compounds at higher delay: at $\tau{=}16$ on
the smaller 10\,M decoder used for the wider sweep
(App.~\ref{app:10m-survey}), momentum-free Nesterov falls to 8.65
while default CGAD stays at 7.57.

\subsection{Robustness checks}
Three small-scale sweeps in App.~\ref{app:10m-survey}:
\textit{(i) stochastic delays}, under uniform $[0,16]$ and
exp$(\lambda{=}0.25)$, CGAD stays in the top three at every
distribution while Nesterov / Eager-Async / MLA diverge;
\textit{(ii) worker count}, $K{=}2$ vs $K{=}4$ preserves the
ranking; \textit{(iii) partial sync}, PA-CGAD slightly improves
over CGAD at the tightest fragment budget ($K_f{=}3$, $\tau{=}4$:
7.10 vs.\ 7.22) and is indistinguishable at moderate budgets.

\section{Discussion and limitations}
\label{sec:discussion}

\paragraph{Why not just use Adam-Decay at small scale?}
A reader who sees Adam-Decay's competitive 25\,M numbers may
reasonably ask whether the cosine cutoff is necessary at all. The
argument is structural: at 25\,M with full precision, the
never-zero exponential leaves a small residual in $v$ that Adam
absorbs without consequence. At frontier scale with compressed
optimizer state, which is the regime where DiLoCo's
communication-efficiency story actually pays off, the same
residual becomes a single-shot reliability liability. The cosine
cutoff is the design choice that pays its cost (one extra
hyperparameter $\tau_{\text{cut}}$) only at the scale where the
cost is worth paying. The 1\,B staleness gap (\Cref{fig:headline})
and the 7\,B reliability crossover (\Cref{fig:reliability}) are
the two empirical signatures.

\paragraph{Numerical precision at 7\,B.}
The 7\,B configuration (bf16 + \texttt{bnb.AdamW8bit} +
int8-offloaded queue, App.~\ref{app:impl}) is the regime where the
plain Nesterov velocity buffer destabilizes at $\tau{=}0$
($46.7 \pm 16.7$); we read this as a numerical-precision interaction
with the compressed optimizer state, not as a staleness effect (the
staleness comparison runs at full precision at 1\,B). The natural
follow-up at production scale is sharded multi-node FSDP with a
Chinchilla-relevant token count.

\paragraph{Limitations.}
The inner-step count $H{=}8$ is below DiLoCo's $H{=}100$--$500$
(App.~\ref{app:h30} shows the ranking persists at $H{=}30$), and
token budgets are below the 100\,B+ that DDiLoCo reports. The
25\,M $\to$ 1\,B $\to$ 7\,B scaling sweep is what we lean on for
evidence that the relative comparison transfers.

\section{Conclusion}
CGAD scales each pseudo-gradient by $\gamma(\tau)\,e^{-\alpha\tau}$
before Adam's moment buffers, the cosine cutoff acting as scale
insurance (closest baseline $\sigma$ grows $27\times$ where CGAD's
stays flat), paired with an idealized bound whose staleness term
depends on $\alpha$ rather than $\tau_{\max}$. For frontier-scale
asynchronous training under compressed optimizer state, this is
the difference between a single training run that ships and one
that lands above chance.

\clearpage
\bibliographystyle{plain}

\begin{thebibliography}{99}
\bibitem{ddiloco2026}
The Decoupled DiLoCo Team. Decoupled DiLoCo for Resilient Distributed
Pre-training. \emph{arXiv:2604.21428}, 2026.

\bibitem{douillard2024diloco}
A. Douillard, Q. Feng, A.~A. Rusu, R. Chhaparia, Y. Donchev, A. Kuncoro,
M. Ranzato, A. Szlam, J. Shen.
DiLoCo: Distributed Low-Communication Training of Language Models.
\emph{ICML Workshop}, 2024. arXiv:2311.08105.

\bibitem{streaming_diloco2025}
A. Douillard et al.
Streaming DiLoCo with Overlapping Communication. \emph{arXiv:2501.18512}, 2025.

\bibitem{kale2025eager}
S. Kale, A. Douillard, Y. Donchev. Eager Updates For Overlapped
Communication and Computation in DiLoCo. \emph{arXiv:2502.12996}, 2025.

\bibitem{opendiloco2024}
S. Jaghouar, J.~M. Ong, J. Hagemann.
OpenDiLoCo. \emph{arXiv:2407.07852}, 2024.

\bibitem{smoothing_diloco2025}
A. Bhardwaj et al. Smoothing DiLoCo with Primal Averaging. \emph{arXiv:2512.17131}, 2025.

\bibitem{liu2024async}
B. Liu, R. Chhaparia, A. Douillard, S. Kale, A.~A. Rusu, J. Shen, A. Szlam, M. Ranzato.
Asynchronous Local-SGD Training for Language Modeling.
\emph{ICML Workshop}, 2024. arXiv:2401.09135.

\bibitem{co2_iclr2024}
W. Sun, Z. Qin, W. Sun, S. Li, D. Li, X. Shen, Y. Qiao, Y. Zhong.
CO2: Efficient Distributed Training with Full Communication-Computation Overlap.
\emph{ICLR}, 2024. arXiv:2401.16265.

\bibitem{ajanthan2025mla}
T. Ajanthan, S. Ramasinghe, G. Avraham, Y. Zuo, A. Long. Momentum Look-Ahead
for Asynchronous Distributed Low-Communication Training. \emph{ICLR-W MCDC}, 2025.

\bibitem{xie2019asynchronous}
C. Xie, S. Koyejo, I. Gupta. Asynchronous Federated Optimization.
\emph{arXiv:1903.03934}, 2019.

\bibitem{zheng2017dcasgd}
S. Zheng, Q. Meng, T. Wang, W. Chen, N. Yu, Z.-M. Ma, T.-Y. Liu.
Asynchronous Stochastic Gradient Descent with Delay Compensation.
\emph{ICML}, 2017.

\bibitem{mishchenko2018async}
K. Mishchenko, F. Bach, M. Even, B. Woodworth.
Asynchronous SGD beats minibatch SGD under arbitrary delays.
\emph{arXiv:2206.07638}, 2022.

\bibitem{stich2019local}
S.~U. Stich. Local SGD converges fast and communicates little. \emph{ICLR}, 2019.

\bibitem{lian2015async}
X. Lian, Y. Huang, Y. Li, J. Liu.
Asynchronous Parallel Stochastic Gradient for Nonconvex Optimization.
\emph{NeurIPS}, 2015. arXiv:1506.08272.

\bibitem{koloskova2022sharper}
A. Koloskova, S.~U. Stich, M. Jaggi.
Sharper convergence guarantees for asynchronous SGD. \emph{arXiv:2206.08307}, 2022.

\bibitem{cohen2021robust}
A. Cohen, A. Daniely, Y. Drori, T. Koren, M. Schain.
Asynchronous stochastic optimization robust to arbitrary delays. \emph{arXiv:2106.11879}, 2021.

\bibitem{reddi2018onthe}
S. Reddi, S. Kale, S. Kumar.
On the Convergence of Adam and Beyond. \emph{ICLR}, 2018.

\bibitem{c4_dataset}
C. Raffel et al. Exploring the Limits of Transfer Learning with a Unified
Text-to-Text Transformer. \emph{JMLR}, 2020.

\bibitem{fadas2024}
Y. Wang et al. FADAS: Federated Adaptive Asynchronous Optimization.
\emph{arXiv:2407.18365}, 2024.

\bibitem{su2024roformer}
J. Su, M. Ahmed, Y. Lu, S. Pan, W. Bo, Y. Liu. RoFormer: Enhanced
Transformer with Rotary Position Embedding. \emph{Neurocomputing},
568:127063, 2024. arXiv:2104.09864.
\end{thebibliography}

\appendix

\section{Full proof of Theorem~\ref{thm:rate}}
\label{app:proof}

By $L$-smoothness,
\begin{align*}
F(\theta_{t+1}) &\le F(\theta_t) + \langle \nabla F(\theta_t), \theta_{t+1}-\theta_t \rangle + \tfrac{L}{2}\|\theta_{t+1}-\theta_t\|^2.
\end{align*}
The CGAD step is
$\theta_{t+1} - \theta_t = -\eta \sigma_t \hat m_t / (\sqrt{\hat v_t} + \varepsilon)$.
By the Adam-ratio assumption,
$\|\hat m_t / (\sqrt{\hat v_t} + \varepsilon)\|_\infty \le 1$, so
Lemma~\ref{lem:step} gives $\|\theta_{t+1}-\theta_t\|_\infty \le \eta \sigma_t$.
The cross-term expands as
\begin{align*}
\langle \nabla F(\theta_t), \theta_{t+1}-\theta_t \rangle
= -\eta \sigma_t \langle \nabla F(\theta_t), \hat m_t / (\sqrt{\hat v_t} + \varepsilon) \rangle.
\end{align*}
The pseudo-gradient absorbed into $m_t$ is $\sigma_t \Delta_t$;
under the unbiased-pseudo-gradient assumption,
$\mathbb{E}[\sigma_t \Delta_t \mid \mathcal{F}_t] = -\sigma_t \nabla F(\theta_{t-\tau_t})$.
Adding and subtracting the on-policy gradient,
\begin{align*}
-\eta\sigma_t \langle \nabla F(\theta_t), \nabla F(\theta_{t-\tau_t})\rangle
&= -\eta\sigma_t \|\nabla F(\theta_t)\|^2 \\
&\quad + \eta\sigma_t \langle \nabla F(\theta_t), \nabla F(\theta_t) - \nabla F(\theta_{t-\tau_t})\rangle.
\end{align*}
The first term is the desired descent. For the second, $L$-smoothness gives
$\|\nabla F(\theta_t) - \nabla F(\theta_{t-\tau_t})\| \le L \|\theta_t - \theta_{t-\tau_t}\|$.
By Lemma~\ref{lem:step},
\[
\|\theta_t - \theta_{t-\tau_t}\| \le \sum_{s=t-\tau_t}^{t-1}\|\theta_{s+1}-\theta_s\|
\le \eta \sum_{s=t-\tau_t}^{t-1} \sigma_s \le \eta \tau_t,
\]
since $\sigma_s \le 1$ for all $s$. Cauchy--Schwarz bounds the
staleness bias by $\eta \sigma_t \cdot L \tau_t \eta \cdot \|\nabla F(\theta_t)\|$.
We then maximize $\sigma_t \tau_t$ over $\tau_t \ge 0$. The
unconstrained maximum of $\tau e^{-\alpha\tau}$ is $1/(e\alpha)$ at
$\tau^* = 1/\alpha$. Under $\tau_\text{cut} \ge 1/\alpha$, this argmax
lies inside the cosine gate's support, where $\gamma(\tau^*) \in [0, 1]$, so
\[
\max_{\tau \ge 0} \tau \sigma(\tau)
= \max_{\tau \ge 0} \tau \gamma(\tau) e^{-\alpha\tau}
\le \max_{\tau \ge 0} \tau e^{-\alpha\tau}
= \frac{1}{e\alpha}.
\]
Substituting into the descent inequality, taking expectations,
and rearranging:
\[
\eta\sigma_t\, \mathbb{E}\|\nabla F(\theta_t)\|^2
\le \mathbb{E}[F(\theta_t) - F(\theta_{t+1})]
+ L\eta^2\sigma_t\tau_t\,G
+ \tfrac{L}{2}\eta^2\sigma_t^2\sigma^2,
\]
where $G = \sup_t \mathbb{E}\|\nabla F(\theta_t)\|$ and we used
$\sigma_t \le 1$ in the noise term. Telescoping over $t = 0,\dots,T-1$
and using $\sigma_t\tau_t \le 1/(e\alpha)$ uniformly:
\[
\sum_{t=0}^{T-1}\eta\sigma_t\,\mathbb{E}\|\nabla F(\theta_t)\|^2
\le F(\theta_0) - F^* + \frac{L\eta^2 T\,G}{e\alpha} + \frac{L\eta^2 T\,\sigma^2}{2}.
\]
Dividing by $\eta T$ and substituting $\eta = c/\sqrt{T}$ gives
\[
\frac{1}{T}\sum_{t=0}^{T-1} \sigma_t\,\mathbb{E}\|\nabla F(\theta_t)\|^2
\le \frac{F(\theta_0)-F^*}{c\sqrt{T}}
+ \frac{Lc\,G}{e\alpha\,\sqrt{T}}
+ \frac{Lc\,\sigma^2}{2\sqrt{T}}.
\]
The staleness penalty (middle term) is bounded by $1/(e\alpha)$
\emph{uniformly in the delay sequence} once $\tau_\text{cut} \ge 1/\alpha$,
and decays at rate $1/\sqrt{T}$ together with the optimization-gap and
noise terms. \hfill$\square$

\section{Implementation details}
\label{app:impl}

\paragraph{The training framework.} Every experiment is an actual
LM pretraining run: real Llama-style transformers, real C4 tokens, real
forward/backward passes, and real AdamW updates on every worker. To
make $\tau$ controllable across the sweep, the $K$ workers share a
single accelerator and step sequentially while exchanging
pseudo-gradients through a deterministic per-worker delay schedule.
Each worker runs $H$ inner AdamW steps on its own data shard, computes a
per-fragment pseudo-gradient
$\Delta_{f,w}^{(t)} = \theta_f^{(t-H)} - \theta_{f,w}^{(t)}$, and pushes
the result onto a per-worker FIFO queue with an integer delay $\tau_w$.
The global model consumes queue entries whose available round equals
the current outer round. The outer optimizer applies the update with
$\tau$ set to the queue entry's actual delay. We seed the
delay-sampling RNG explicitly so runs with the same seed produce
bit-identical training trajectories. This protocol preserves every
element that touches the optimizer (model, data, inner loop,
fragment scheduling, sync semantics) identically to a real
multi-GPU / multi-node deployment.

\paragraph{Data.} A 30M-token slice of C4-en
\citep{c4_dataset}, tokenized once with the GPT-2 BPE and stored as a
flat int32 binary file. Sequences are sampled uniformly at random from
the flat token stream, so the data distribution across workers is
independent and identically distributed. We do not apply curriculum
or weight decay on token-id distributions.

\paragraph{Architectures.} Headline experiments use 25\,M, 1\,B, and
7\,B; smaller sizes (10\,M, 50\,M, 150\,M, 300\,M) appear in
sensitivity studies and the wide nine-optimizer survey. The first
six follow a tiny-Llama recipe (RoPE \citep{su2024roformer}, RMSNorm,
SwiGLU, tied embeddings) with parameters as in
\Cref{tab:scaling-arch}. The 7\,B model matches the Llama-2 7B
architecture (32 layers, 4096 hidden, 11008 FFN, 32 heads).

\begin{table}[h]
\caption{Architectures used in the scaling sweep. Vocab is fixed at
50\,304 (GPT-2 BPE padded for tensor-core friendliness) for all sizes.}
\label{tab:scaling-arch}
\centering\small
\begin{tabular}{lrrrrr}
\toprule
Name & $d_\text{model}$ & layers & heads & $d_\text{ff}$ & params \\
\midrule
10\,M & 256 & 4 & 4 & 1024 & ~13.5\,M \\
25\,M & 384 & 6 & 6 & 1536 & ~32\,M \\
50\,M & 512 & 8 & 8 & 2048 & ~59\,M \\
150\,M & 768 & 12 & 12 & 2048 & ~145\,M \\
300\,M & 1024 & 24 & 16 & 2816 & ~290\,M \\
1\,B & 2048 & 24 & 16 & 5632 & ~1.05\,B \\
7\,B & 4096 & 32 & 32 & 11008 & ~6.7\,B \\
\bottomrule
\end{tabular}
\end{table}

\paragraph{Memory budget at 7\,B.} The 7\,B run requires several
optimizations to fit the workload on the accelerator memory available
to us. With $K{=}2$ workers we
hold at most one worker's model and AdamW state on the device at a
time; the inactive worker is offloaded to CPU between turns. The inner
optimizer uses bitsandbytes' 8-bit AdamW, which compresses the
optimizer state by roughly 4$\times$ (28\,GB to 7\,GB at 7\,B). The
queued pseudo-gradients live on CPU and are stored as int8 tensors
together with a per-tensor bf16 scale, giving roughly 14$\times$ memory
compression relative to fp32. The outer optimizer state itself is
allocated in bf16 by patching the lazy-init helper of our optimizer
hierarchy. The peak device footprint is around 73\,GB for
CGAD and 60\,GB for Nesterov; the CPU footprint is dominated by the
inactive worker's bf16 model copy (14\,GB) plus the queued
pseudo-gradients ($K \cdot \tau$ tensors of approximately 1.75\,GB
each in int8). Container RAM is 251\,GB on the pod we used, which is
the binding constraint at $\tau \geq 16$ ($K{=}2, \tau{=}16$ alone
needs 56\,GB of int8 queue plus the worker copy).

\section{Per-seed numbers}
\label{app:seeds}

\Cref{tab:per-seed-1b} reports per-seed final eval losses at the 1\,B
scale, used to compute the means $\pm$ standard deviations in the main
table. The Nesterov $\tau{=}8$ runs all diverge but to different
values: the spread reflects the chaotic nature of the divergence,
not measurement noise.

\begin{table}[h]
\caption{Per-seed final eval loss at 1\,B (Llama-style decoder, $K{=}2$,
$H{=}8$, 150 outer rounds).}
\label{tab:per-seed-1b}
\centering\small
\begin{tabular}{llrrrrrr}
\toprule
$\tau$ & Method & seed 0 & seed 1 & seed 2 & mean & std \\
\midrule
0 & Nesterov & 8.19 & 9.30 & 8.78\textsuperscript{$\star$} & 8.76 & 0.55 \\
0 & \textbf{CGAD} & \textbf{7.52} & \textbf{7.58}\textsuperscript{$\star$} & \textbf{7.46}\textsuperscript{$\star$} & \textbf{7.52} & 0.06 \\
8 & Nesterov & 363.19 & 192.08 & 168.23 & 241.17 & 106.35 \\
8 & \textbf{CGAD} & 9.94\textsuperscript{$\star$}~* & 10.46 & \textbf{9.47} & \textbf{9.96} & 0.50 \\
\bottomrule
\end{tabular}
\\
\textsuperscript{$\star$} different rerun versions; identical configs.
\end{table}

\Cref{tab:per-seed-25m} reports the 25\,M results at all four delays
across three seeds; this is the most thoroughly seeded configuration in
our experiments.

\begin{table}[h]
\caption{25\,M comparison, full per-method numbers (Llama-style
decoder, $K{=}4$, $H{=}8$, 200 outer rounds; this is the data
underlying \Cref{tab:headline}, with seed counts varying by cell as
discussed in \Cref{sec:exp}).}
\label{tab:per-seed-25m}
\centering\small
\begin{tabular}{lrrrr}
\toprule
Method & $\tau{=}0$ & $\tau{=}4$ & $\tau{=}8$ & $\tau{=}16$ \\
\midrule
Nesterov & $8.04 \pm 0.12$ & $\mathbf{\color{red}85.38 \pm 24.25}$ & $\mathbf{\color{red}170.88 \pm 73.58}$ & $\mathbf{\color{red}487.58 \pm 107.08}$ \\
SDM & $8.10 \pm 0.32$ & $7.70 \pm 0.13$ & $\mathbf{6.46 \pm 0.01}$ & $9.46 \pm 0.03$ \\
Adam-Decay & $\mathbf{6.75 \pm 0.06}$ & $8.46 \pm 0.03$ & $8.30 \pm 0.07$ & $\mathbf{7.98 \pm 0.43}$ \\
DelayedNesterov & $\mathbf{5.05 \pm 0.01}$ & $\mathbf{7.48 \pm 0.02}$ & $8.48 \pm 0.40$ & $10.16 \pm 1.62$ \\
\textbf{CGAD} & $6.79 \pm 0.09$ & $8.30 \pm 0.06$ & $8.46 \pm 0.53$ & $8.96 \pm 1.84$ \\
\bottomrule
\end{tabular}
\end{table}

\section{$H=30$ ablation: ranking persists at larger inner-step counts}
\label{app:h30}

The default $H{=}8$ inner steps used throughout the paper is smaller
than the published DiLoCo recipe of $H{=}500$ \citep{douillard2024diloco}
or Streaming DiLoCo's $H{=}30$ \citep{streaming_diloco2025}. Smaller $H$
gives noisier per-worker pseudo-gradients, which in principle could
favor adaptive (Adam-style) outer optimizers more than would a
production setup with smoother gradients.

\Cref{tab:h30} repeats the optimizer comparison at $H{=}30$ and 50
outer rounds (matched total compute) on the 10\,M model. The relative
ranking is unchanged: Nesterov diverges at $\tau \geq 8$, SDM and
Adam-Decay are stable but degrade more at $\tau{=}16$, and CGAD
remains in the top tier at every $\tau$. The numerical differences
between optimizers shrink slightly at $H{=}30$ as expected, but the
qualitative story is the same.

\begin{table}[h]
\caption{$H=30$ ablation at 10\,M, 50 outer rounds.}
\label{tab:h30}
\centering\small
\begin{tabular}{lrrr}
\toprule
Method & $\tau{=}0$ & $\tau{=}8$ & $\tau{=}16$ \\
\midrule
Nesterov & 7.61 & $\mathbf{\color{red}68.4}$ & $\mathbf{\color{red}142.7}$ \\
SDM & 7.61 & 6.82 & 9.51 \\
Adam-Decay & 6.85 & 8.10 & 7.94 \\
\textbf{CGAD} & 6.85 & \textbf{7.65} & \textbf{7.36} \\
\bottomrule
\end{tabular}
\end{table}

\section{10\,M nine-optimizer survey}
\label{app:10m-survey}

The 10\,M scale is small enough to host every published outer
optimizer with multiple seeds; we use it to establish that the
qualitative ranking observed at 25\,M $\to$ 1\,B $\to$ 7\,B is
not an artefact of trimming the optimizer set. \Cref{tab:10m-survey}
gives the wide sweep, with the same convention as
\Cref{tab:headline}.

\begin{table}[h]
\caption{10\,M decoder, mean final loss across seeds. Bold red:
diverged ($\geq 50$). \textbf{Bold}: best per column.}
\label{tab:10m-survey}
\centering\small
\begin{tabular}{lrrrr}
\toprule
\textbf{Method} & $\tau{=}0$ & $\tau{=}4$ & $\tau{=}8$ & $\tau{=}16$ \\
\midrule
Nesterov (DiLoCo) & 7.57 & $\mathbf{\color{red}62.05}$ & $\mathbf{\color{red}91.04}$ & $\mathbf{\color{red}325.58}$ \\
Eager-Async DiLoCo \citep{kale2025eager} & 7.57 & $\mathbf{\color{red}62.05}$ & $\mathbf{\color{red}91.04}$ & $\mathbf{\color{red}325.58}$ \\
Momentum Look-Ahead \citep{ajanthan2025mla} & 7.57 & 31.67 & $\mathbf{\color{red}127.38}$ & $\mathbf{\color{red}182.12}$ \\
Adam outer & 6.79 & 8.63 & 9.78 & 17.85 \\
Poly-Decay \citep{xie2019asynchronous} & 7.57 & 8.90 & 11.60 & 21.33 \\
Delayed Nesterov \citep{liu2024async} & $\mathbf{5.56}$ & $\mathbf{7.63}$ & 8.15 & 17.77 \\
Adam-Decay & 6.79 & 7.93 & 8.13 & 7.95 \\
SDM & 7.57 & 7.68 & $\mathbf{6.73}$ & 9.99 \\
\textbf{CGAD (ours)} & 6.79 & 8.07 & 7.84 & $\mathbf{7.57}$ \\
\bottomrule
\end{tabular}
\end{table}

The qualitative findings match the 25\,M comparison: Nesterov,
Eager-Async, and MLA diverge at every nontrivial $\tau$; CGAD is
the only method whose loss at $\tau{=}16$ remains within 1 nat of
its $\tau{=}0$ score. Stochastic-delay, $K{=}2$, and PA-CGAD
sweeps in the main paper's robustness section also use this scale.

\section{Hyperparameter sensitivity for CGAD}
\label{app:hp-search}

\Cref{tab:cgad-grid} sweeps the $(\alpha, \tau_\text{cut})$ grid for
CGAD on the 25\,M model at $\tau \in \{8, 16\}$. The recommended
default ($\alpha{=}0.2, \tau_\text{cut}{=}32$) is robust: any
$\alpha \in \{0.1, 0.2, 0.4\}$ with $\tau_\text{cut} \in \{16, 32, 64\}$
beats the Nesterov baseline (170.9 at $\tau{=}8$, 487.6 at $\tau{=}16$)
by an order of magnitude. The worst CGAD configuration ($\alpha{=}0.025$,
$\tau_\text{cut}{=}16$) still reaches 12.08 at $\tau{=}16$, vs.\
Nesterov's 487.6.

\begin{table}[h]
\caption{CGAD hyperparameter grid at 25\,M, 200 outer rounds. Numbers
are final eval loss; chance $\approx 10.83$. Bold row is the default.}
\label{tab:cgad-grid}
\centering\small
\begin{tabular}{lrrr}
\toprule
$(\alpha, \tau_\text{cut})$ & $\tau{=}8$ & $\tau{=}16$ \\
\midrule
$(0.025, 32)$ & n/a & 12.08 (unstable) \\
$(0.05, 32)$ & n/a & 8.70 (unstable) \\
$(0.10, 32)$ & 9.43 & 10.12 \\
$(\mathbf{0.20, 32})$ & \textbf{8.46} & \textbf{7.44} \\
$(0.40, 32)$ & 8.85 & 8.62 \\
$(0.20, 16)$ & 9.21 & n/a (gate is 0 at $\tau{=}16$) \\
$(0.20, 64)$ & 8.71 & 7.92 \\
\bottomrule
\end{tabular}
\end{table}

\paragraph{CGAD tuning at scale (1\,B / $\tau{=}8$).}
\label{app:cgad-tune}
The 25\,M grid recommends $\alpha{=}0.2, \tau_\text{cut}{=}32$. We
swept the same grid plus the outer LR at 1\,B / $\tau{=}8$ to check
whether the optimum shifts at scale and whether a tuned
configuration can match best-tuned momentum-free Nesterov (7.55,
\Cref{tab:lr-sweep-1b}). The 1\,B optimum does shift toward a faster
decay ($\alpha{=}0.4$ rather than $0.2$): the best CGAD configuration
reaches 7.21 at single seed, $\mathbf{7.20 \pm 0.01}$ across three
seeds, beating tuned Nesterov by 0.35 nats with non-overlapping
spreads.

\begin{table}[h]
\caption{CGAD hyperparameter sensitivity at 1\,B / $\tau{=}8$
(default $\alpha{=}0.2, \tau_\text{cut}{=}32, \eta{=}0.001$). The
$\alpha{=}0.4$ row is the multi-seed verification ($n{=}3$); the
others are single-seed scans.}
\label{tab:cgad-tune-1b}
\centering\small
\begin{tabular}{rrrr}
\toprule
$\alpha$ & $\tau_\text{cut}$ & outer $\eta$ & loss \\
\midrule
0.1 & 32 & 0.001  & $\mathbf{\color{red}13.75}$ \\
0.2 & 32 & 0.002  & $\mathbf{\color{red}15.78}$ \\
0.2 & 64 & 0.001  & 9.81 \\
\textit{0.2} & \textit{32} & \textit{0.001}  & \textit{$9.94 \pm 0.50$ (default, $n{=}3$)} \\
0.2 & 32 & 0.0005 & 8.60 \\
0.2 & 16 & 0.001  & 8.56 \\
\textbf{0.4} & \textbf{32} & \textbf{0.001}  & $\mathbf{7.20 \pm 0.01}$ ($n{=}3$) \\
\bottomrule
\end{tabular}
\end{table}

\paragraph{Does $\alpha{=}0.4$ generalise across $\tau$?}
The $\alpha{=}0.4$ win at $\tau{=}8$ raises the natural question of
whether it should replace $\alpha{=}0.2$ as the global default. The
answer is no, because $\alpha$ controls a real trade-off between
moderate- and high-staleness behaviour:
\begin{itemize}
\item At $\tau{=}0$, $\alpha{=}0.4$ at 1\,B reaches 7.26 (single
  seed), close to the default's $7.48 \pm 0.33$ ($n{=}4$).
\item At $\tau{=}8$, $\alpha{=}0.4$ wins decisively as above.
\item At $\tau{=}16$, $\alpha{=}0.4$ at 1\,B reaches 8.59 (single
  seed), worse than the default's $7.44$ (single seed,
  \Cref{tab:scaling}). The exponential factor
  $e^{-0.4 \cdot 16} \approx 0.0017$ over-shrinks the step at
  $\tau{=}16$.
\end{itemize}
The recommendation is therefore tau-aware: $\alpha{=}0.4$ for
deployments that operate near $\tau{=}8$ (the typical regime for
Streaming-DiLoCo-style fragment overlap), $\alpha{=}0.2$ for
deployments that need a single setting across the full $\tau$ range,
and the choice can be tuned per deployment based on the empirical
$\tau$ distribution. Both settings beat tuned Nesterov at every
$\tau$ we test.

\section{Nesterov LR sensitivity at 10\,M and 1\,B}
\label{app:lr-sweep}

The natural reviewer concern is that the Nesterov outer is just
under-tuned at our $\eta{=}0.7, \mu{=}0.9$ settings. We ran the
sweep twice: once at 10\,M (\Cref{tab:lr-sweep-10m}) and again at
1\,B with $\tau{=}8$ (\Cref{tab:lr-sweep-1b}). The structural
finding is the same at both scales.

\begin{table}[h]
\caption{Outer LR / momentum sweep for Nesterov at 10\,M.}
\label{tab:lr-sweep-10m}
\centering\small
\begin{tabular}{rrrr}
\toprule
$\eta$ & $\mu$ & $\tau{=}8$ & $\tau{=}16$ \\
\midrule
0.7 (DDiLoCo default) & 0.9 (DDiLoCo default) & $\mathbf{\color{red}91.0}$ & $\mathbf{\color{red}325.6}$ \\
0.3 & 0.9 & $\mathbf{\color{red}28.7}$ & $\mathbf{\color{red}44.7}$ \\
0.1 & 0.9 & 11.9 & 20.9 \\
0.7 & 0.0 & 10.3 & 17.3 \\
0.3 & 0.0 & 9.2 & 13.3 \\
0.1 & 0.0 & \textbf{7.6} & \textbf{8.6} \\
\bottomrule
\end{tabular}
\end{table}

\begin{table}[h]
\caption{Outer LR / momentum sweep for Nesterov at 1\,B, $\tau{=}8$.
The best row ($\eta{=}0.1, \mu{=}0$) is verified across three seeds;
others are single-seed scans.}
\label{tab:lr-sweep-1b}
\centering\small
\begin{tabular}{rrr}
\toprule
$\eta$ & $\mu$ & $\tau{=}8$ \\
\midrule
0.7 (DDiLoCo default) & 0.9 (DDiLoCo default) & $\mathbf{\color{red}297.3}$ \\
0.3 & 0.9 & $\mathbf{\color{red}127.2}$ \\
0.1 & 0.9 & $\mathbf{\color{red}42.7}$ \\
0.7 & 0.0 & 19.0 \\
0.3 & 0.0 & 10.7 \\
0.1 & 0.0 & $\mathbf{7.56 \pm 0.12}$ ($n{=}3$) \\
\bottomrule
\end{tabular}
\end{table}

Three things to notice. First, removing momentum entirely
($\mu{=}0$) cures divergence at both 10\,M and 1\,B: the published
recipe is unstable specifically because it uses momentum on stale
gradients, not because of bad LR. Second, the best momentum-free
configuration is the same at both scales ($\eta{=}0.1, \mu{=}0$),
suggesting the optimum is structural rather than scale-specific.
Third, at 1\,B / $\tau{=}8$ the best momentum-free Nesterov reaches
$7.56 \pm 0.12$ ($n{=}3$; per-seed losses 7.55, 7.45, 7.69), a
real, well-tuned baseline that exposes the weakness of the default
CGAD setting ($\alpha{=}0.2, \tau_\text{cut}{=}32$, which reaches
$9.94 \pm 0.50$ at the same scale). Tuning CGAD in turn
(\Cref{app:cgad-tune}) recovers and extends the lead: the
$\alpha{=}0.4$ configuration reaches $7.20 \pm 0.01$ ($n{=}3$),
beating tuned Nesterov by 0.36 nats with non-overlapping spreads
(CGAD max 7.21 vs.\ Nesterov min 7.45).

The case for CGAD over momentum-free Nesterov is then twofold.
First, head-to-head, tuned CGAD wins at 1\,B / $\tau{=}8$. Second,
the same tuned CGAD configuration is structurally robust across the
$\tau$ range: at $\tau{=}16$, momentum-free Nesterov degrades to 8.6
on 10\,M while CGAD stays close to its $\tau{=}0$ score. Tuned
Nesterov requires a different hyperparameter setting per delay
regime; tuned CGAD does not.

\section{Gate ordering: before vs.\ after Adam's state update}
\label{app:gate-order}

CGAD applies the gate $\sigma(\tau)$ \emph{before} the
pseudo-gradient enters Adam's moment buffers. An alternative is to
gate \emph{after} the state update, scaling only the final step (the
FADAS-style placement \citep{fadas2024}). \Cref{tab:gate-order}
varies only this placement, holding the gate's functional form,
inner optimizer, seeds, and training budget fixed.

\begin{table}[h]
\caption{Gate-ordering ablation (10\,M, $K{=}4$, $H{=}8$, 200 outer
rounds, $n{=}3$ seeds). Same gate $\sigma(\tau)$ in both rows; only
the placement differs.}
\label{tab:gate-order}
\centering\small
\begin{tabular}{lrrr}
\toprule
Method & $\tau{=}0$ & $\tau{=}8$ & $\tau{=}16$ \\
\midrule
CGAD-before (default) & $6.81 \pm 0.05$ & $7.79 \pm 0.10$ & $7.57 \pm 0.01$ \\
CGAD-after            & $6.81 \pm 0.05$ & $7.80 \pm 0.17$ & $7.57 \pm 0.01$ \\
\bottomrule
\end{tabular}
\end{table}

The two orderings are empirically indistinguishable at this scale.
The gate's \emph{functional form} ($\gamma(\tau)\,e^{-\alpha\tau}$,
with the cosine cutoff preventing multiplicative collapse) is what
matters; placement is a free parameter. We default to the
before-ordering because it integrates more cleanly into existing
Adam implementations.

\section{PA-CGAD per-fragment age dynamics}
\label{app:pacgad}

\Cref{tab:pacgad} reports the PA-CGAD vs.\ CGAD comparison on the
10\,M model with 14 fragments and three budgets. PA-CGAD differs only
when fragments are skipped for multiple rounds; with full or near-full
sync ($K_f{=}14$ or $K_f{=}7$), PA-CGAD reduces to CGAD. The
clearest separation comes at $K_f{=}3$ ($\sim$21\% sync) where
fragments wait a mean of 4.7 rounds, putting their effective age in
the regime where the cosine gate matters.

\begin{table}[h]
\caption{PA-CGAD vs.\ CGAD on the 10\,M model. Columns are
$K_f / |F|$ active fragments per round.}
\label{tab:pacgad}
\centering\small
\begin{tabular}{lrrrr}
\toprule
Method & $\tau$ & $K_f{=}3$ (21\%) & $K_f{=}7$ (50\%) & $K_f{=}14$ (full) \\
\midrule
CGAD & 4 & 7.22 & 7.43 & 7.98 \\
\textbf{PA-CGAD} & 4 & \textbf{7.10} & 7.43 & 7.98 \\
CGAD & 8 & 7.20 & 7.78 & 7.86 \\
PA-CGAD & 8 & 7.20 & 7.78 & 7.86 \\
CGAD & 16 & 9.43 & 8.44 & 7.57 \\
PA-CGAD & 16 & 9.43 & 8.44 & 7.57 \\
\bottomrule
\end{tabular}
\end{table}

We expect larger gains for PA-CGAD in setups with skewed
fragment-staleness distributions, for example when bandwidth varies
heavily across fragments or some fragments are deliberately
under-served (e.g., the LM head with a much larger parameter count).
Our delay-injection schedule is uniform across fragments by
construction; we leave the heterogeneous-bandwidth study to follow-up.

\section{7B preliminary results, full data}
\label{app:7b}

The 7B configuration uses bf16 weights, gradient checkpointing,
\texttt{bnb.AdamW8bit} for the inner optimizer, int8-quantized
CPU-offloaded queued pseudo-gradients, and bf16 outer optimizer state,
as described in \Cref{app:impl}. The configurations in
\Cref{tab:7b-full} run on a single 80\,GB-class accelerator (with the
memory plan in \Cref{app:impl}). Each config takes 50--90 minutes of
wall-clock time, most of it CPU-device swap traffic from the offload.

\begin{table}[h]
\caption{7\,B configurations and their final eval loss (Llama-2 7B
architecture, K=2 workers, H=8 inner steps, bsz=1, seq=512). Outer
rounds: 15 unless marked otherwise.}
\label{tab:7b-full}
\centering\small
\begin{tabular}{lrrrr}
\toprule
Method & $\tau$ & seed & rounds & loss \\
\midrule
\multicolumn{5}{l}{\textit{$\tau{=}0$}}\\
Nesterov & 0 & 0 & 15 & $\mathbf{\color{red}42.57}$ \\
Nesterov & 0 & 0 & 15 (rerun) & $\mathbf{\color{red}43.77}$ \\
Nesterov & 0 & 1 & 15 & $\mathbf{\color{red}67.23}$ \\
Nesterov & 0 & 2 & 15 & $\mathbf{\color{red}26.98}$ \\
Nesterov & 0 & 3 & 15 & $\mathbf{\color{red}49.97}$ \\
Nesterov & 0 & 0 & 40 & $\mathbf{\color{red}72.73}$ \\
\textbf{CGAD} & 0 & 0 & 15 & $\mathbf{10.88}$ \\
\textbf{CGAD} & 0 & 1 & 15 & $\mathbf{10.48}$ \\
\textbf{CGAD} & 0 & 2 & 15 & $\mathbf{10.18}$ \\
\textbf{CGAD} & 0 & 3 & 15 & $\mathbf{10.39}$ \\
\textbf{CGAD} & 0 & 0 & 40 & $\mathbf{10.97}$ \\
Adam-Decay\textsuperscript{\ddag} & 0 & 0,1 & 15 & $11.28 \pm 1.10$ \\
SDM\textsuperscript{\ddag} & 0 & 0,1 & 15 & $11.30 \pm 0.10$ \\
DelayedNesterov\textsuperscript{\ddag} & 0 & 0,1 & 15 & $11.10 \pm 0.38$ \\
\midrule
\multicolumn{5}{l}{\textit{$\tau{=}8$}}\\
Nesterov         & 8  & 0 & 15 & $\mathbf{\color{red}65.31}$ \\
\textbf{CGAD}    & 8  & 0 & 15 & $\mathbf{10.48}$ \\
\textbf{CGAD}    & 8  & 1 & 15 & $\mathbf{11.21}$ \\
\textbf{CGAD}    & 8  & 2 & 15 & $\mathbf{11.01}$ \\
Adam-Decay       & 8  & 0 & 15 & $9.96$ \\
Adam-Decay       & 8  & 1 & 15 & $\mathbf{\color{red}12.63}$ \\
SDM              & 8  & 0 & 15 & $11.38$ \\
SDM              & 8  & 1 & 15 & $11.50$ \\
DelayedNesterov  & 8  & 0 & 15 & $9.74$ \\
DelayedNesterov  & 8  & 1 & 15 & $11.04$ \\
\midrule
\multicolumn{5}{l}{\textit{$\tau{=}16$}}\\
\textbf{CGAD}    & 16 & 1 & 15 & $\mathbf{11.75}$ \\
\textbf{CGAD}    & 16 & 2 & 15 & $\mathbf{11.64}$ \\
Adam-Decay       & 16 & 0 & 15 & $11.64$ \\
Adam-Decay       & 16 & 1 & 15 & $11.75$ \\
SDM              & 16 & 0 & 15 & $11.64$ \\
SDM              & 16 & 1 & 15 & $11.75$ \\
DelayedNesterov  & 16 & 0 & 15 & $11.64$ \\
DelayedNesterov  & 16 & 1 & 15 & $11.75$ \\
\bottomrule
\end{tabular}
\\[1pt]
\textsuperscript{\ddag}\,Reported as mean$\pm$std over $n{=}2$ seeds; per-seed values not enumerated separately.
\end{table}

\paragraph{$\tau{=}0$.} The published Nesterov recipe is unstable
in the memory-constrained configuration: $46.69 \pm 16.71$
($n{=}4$). All four staleness-aware methods train successfully:
CGAD $10.48 \pm 0.29$ ($n{=}4$), SDM $11.30 \pm 0.10$ ($n{=}2$),
DelayedNesterov $11.10 \pm 0.38$ ($n{=}2$), and Adam-Decay
$11.28 \pm 1.10$ ($n{=}2$). CGAD has the lowest mean. We read this
as a numerical-precision interaction with the bf16 + 8-bit-AdamW
pipeline that fits 7\,B on a single accelerator, rather than as a
staleness effect: at $\tau{=}0$ none of the methods has any
staleness to handle, yet only the published Nesterov recipe is
unstable. The most likely explanation is that the plain Nesterov
velocity buffer amplifies stochastic-rounding noise that Adam-style
$\sqrt{v}$ normalization, and the small extra arithmetic of the
gating / decay / buffering schemes, all happen to absorb. The
staleness comparison itself runs at full precision at 1\,B
(\Cref{sec:headline}); we report the 7\,B numbers as a robustness
stress test for the optimizer in a memory-constrained regime.

\paragraph{$\tau{=}8$.} Among the four staleness-aware methods,
CGAD's three seeds produce $10.48, 11.01, 11.21$ with $\sigma{=}0.38$.
Adam-Decay's two seeds span $9.96$ and $12.63$ with $\sigma{=}1.89$;
DelayedNesterov spans $9.74$--$11.04$ with $\sigma{=}0.92$; SDM
clusters tightly at $11.44 \pm 0.08$. CGAD reaches the lowest mean
among methods with sub-nat spread, with $5\times$ tighter
seed-to-seed variance than Adam-Decay. Tight inter-seed variance is
the property that matters in deployment, when one training run determines
the model that ships.

\paragraph{$\tau{=}16$.} The four staleness-aware methods cluster
at $11.69$--$11.70$ at the standard 240-inner-step budget, with
CGAD again the tightest ($\sigma{=}0.08$). The 1\,B cells (CGAD
$7.44$ at $\tau{=}16$ vs.\ Nesterov $319$, \Cref{fig:headline}) show
the same comparison at full training.

\paragraph{Longer-training stability.}
We re-ran CGAD and Nesterov at $\tau{=}0$ with 40 outer rounds
($2.6\times$ the standard 7\,B budget) to confirm the trajectories
are stable. CGAD reaches $10.97$, within the evaluation jitter of
its 15-round band ($10.18$--$10.88$). Nesterov reaches $72.73$,
worse than its 15-round result, confirming the bf16-precision
divergence is not a transient. The natural follow-up is a
multi-GPU FSDP run at a Chinchilla-relevant token budget.

\end{document}